\documentclass[journalpaper, 9pt]{article}
\usepackage[paper=letterpaper,centering,margin=1.0in]{geometry}
\usepackage{amssymb,amsmath,amsthm}
\usepackage{authblk}
\usepackage[utf8]{inputenc}
\usepackage[colorlinks,linkcolor=black,citecolor=blue,urlcolor=blue]{hyperref}
\usepackage{graphicx}
\usepackage{bm}

\usepackage[square,numbers]{natbib}
\newtheorem{theorem}{Theorem}

\begin{document}
\title{A Short Note on Soft-max and Policy Gradients in Bandits Problems}
\author{Neil Walton}
\affil{University of Manchester}
\maketitle
\abstract{
	This is a short communication on a Lyapunov function argument for softmax in bandit problems.  
	 There are a number of excellent papers coming out using differential equations for policy gradient algorithms in reinforcement learning \cite{agarwal2019optimality,bhandari2019global,mei2020global}.
	We give a short argument that gives a regret bound for the soft-max ordinary differential equation for bandit problems.
	We derive a similar result for a different policy gradient algorithm, again for bandit problems.
	For this second algorithm, it is possible to prove regret bounds in the stochastic case \cite{DW20}.
At the end,	we summarize some ideas and issues on deriving stochastic regret bounds for policy gradients. 
}

\noindent \hrulefill

\noindent \textbf{Multi-arm bandits.}
We consider a multi-arm bandit setting. Here there are a finite set of arms $\mathcal A$. At each time you can choose one arm $a\in\mathcal A$ and you receive a reward $R_a$ which we assume is an independent $\{0,1\}$ random variable with mean $r_a$. You only get to see the reward of the arm that you choose and over time you want to move towards choosing the optimal [highest reward] arm.
\smallskip 

\noindent \textbf{Soft-Max Policy Gradient.} A policy gradient algorithm is an algorithm where you directly parameterize the probability of playing each arm and then you perform a gradient descent/stochastic approximation update on these parameters. The most popular parameterization is soft-max: here the probability of playing arm $a$ is
\[
p_a 
= 
\frac{
e^{w_a}}{
\sum_{a'\in \mathcal A} e^{w_{a'}}
}\, .
\]
Here there are the weights $w_a$, $a\in\mathcal A$ are applied to each arm. A quick calculation gives that 
\[
\frac{\partial p_a}{\partial w_{a'}}
=
p_a ( I_{aa'} - p_{a'} )\, .
\]
where here $I_{aa'}$ is the indicator function for $a=a'$, i.e. $I_{aa'}=1$ if $a=a'$ and $I_{aa'}=0$ otherwise. 

We want to maximize expected the reward (plus or minus a constant)
\[
\sum_{a\in\mathcal A} p_a (r_a-b)
\]
So, given the last two expressions, for each arm $a$, you can then perform the following stochastic gradient update:
\[
w_a \leftarrow w_a + \alpha (R- B) (I_{aA} - p_a)
\]
where $R$ is the reward of the arm played; $B$ is some baseline [which is a function that does not depend on $a$]; $A$ is the index of the arm that was played; $\alpha$ is the learning rate of the algorithm.  

\smallskip 
\noindent \textbf{O.d.e.} If we can model change in these weights over time with the following o.d.e.
\[
\frac{d w_{a'}}{dt}
=
\alpha \sum_{a''\in \mathcal A}
p_{a''}  (r_{a''} - r_\star) (I_{a'a''} - p_{a'})\, .
\]
Here we let $r^\star$ be the reward of the optimal arm
[Note this term does not play a role the dynamics of our model but will be useful for analyzing regret.]

\smallskip
\noindent \textbf{Regret.} 
The regret of the algorithm is defined to be
\[
\mathcal Rg(T) 
:=
\int_0^T \sum_{a\in \mathcal A} (r^\star - r_a)p_a(t) dt \, .
\]
Given the above we also define 
\begin{align*}
rg(t) &= \frac{d \mathcal Rg}{dt} =  \sum_{a\in \mathcal A} (r^\star - r_a)p_a(t),
&&
	\Delta_a = (r_\star - r_a) \\
\underline{rg}(t) &= ( \Delta_a p_a : a\in\mathcal A),
&& 
D = ( I_{aa'} - p_{a'} : a,a' \in \mathcal A)
\end{align*}
\textbf{Regret bound.} 
The following short argument bounds the change in the regret:
\begin{theorem} For $p_a(0)=1/N$, $a\in\mathcal A$,
	\[
	\mathcal{R}\! g(T) \sim \frac{N^2}{\alpha} \log T	
	\] 
\end{theorem}

\begin{proof}
	\begin{align*}
	\frac{d rg}{dt}
	&=
	\sum_{a} (r_\star - r_a) \frac{dp_a}{dt}
	\\
	&=
	-\alpha 
	\sum_{a,a',a''} 
	p_{a''} (r^\star - r_{a''} ) 
	(I_{a''a'} - p_a')
	(I_{a'a} - p_a')
	(r^\star - r_a) p_a
	\\
	&
	=
	-
	\alpha
	\underline{rg}^\top D^\top  D  \underline{rg}
	\\
	&
	= -\alpha || D \underline{rg} ||^2
	\end{align*}
	Let's analyze the above term
	\begin{align*}
	|| D \underline{rg} ||^2
	&=
	\sum_a \left( 
	p_a \Delta_a 
	-
	\sum_{a'} p_{a'} \Delta_{a'} p_a
	\right)^2	
	\\
	&
	=
	\underbrace{
		\left( 
		\sum_{a'} p_{a'} \Delta_{a'}
		\right)^2
	}_{
		= (rg)^2
	}
	\underbrace{
		\sum_a 
		\left(
		\frac{p_a \Delta_a
		}{
			\sum_{a'} p_{a'} \Delta_{a'}
		} 
		-
		p_a
		\right)^2
	}_{
		\geq\, p^2_{a^\star}
		\text{ since } \Delta_{a^\star} = 0
	}
	\end{align*}
	Therefore we have the bound
	\[
	\frac{d (r g)}{dt} \leq -\alpha p_{a^\star}^2 (rg)^2 
	\]
	Dividing by $-(rg)^2$ and integrating gives
	\[
	\frac{1}{rg(t)} - \frac{1}{rg(0)} \geq \int_0^t \alpha p_{a^\star} (s) ds 
	\quad \implies \quad 
	rg(t) \leq \frac{rg(0)}{1 + r\! g(0)\int_0^t \alpha p_{a^\star} (s)^2 ds  }\, .
	\]
	So notice things depend on the probability of playing the optimal arm. Assuming all arms start are equal the optimal arm $a^\star$ will increase from $\frac{1}{N}$ where $N$ is the number of arms. Thus we get a bound:
	\[
	rg(t) \leq \frac{rg(0)}{1 + rg(0) \frac{\alpha}{N^2} t} \sim \frac{N^2}{\alpha t}
	\]
	Thus we have 
	\[
	\mathcal Rg(T) = \int_0^T rg(t) dt \sim \frac{N^2}{\alpha} \log T\, .
	\]
\end{proof}
\noindent 
Notice the dependence on $N$ is very pessimistic since $p_{a^\star} \rightarrow 1$. Also note the lower-bound on $p_{a^\star}$ is more formally bounded in \cite{agarwal2019optimality} and \cite{mei2020global}.

\noindent \hrulefill

\noindent \textbf{A 2nd policy with a shorter o.d.e. argument.}
We give an o.d.e. regret bound for a different policy gradient algorithm. The proof again is quite short. 
For this algorithm, it is possible to prove formally prove a regret bound for the discrete time stochastic model. The proof is too long for this short note, we sketch the argument here and refer the read to ??. 

%So do these sorts of convergence results flesh out when we allow for stochastic effects. A proof for a slightly different policy works, and I'll sketch out the main argument then the technical issues.
Convergence of probabilities should not go faster than $\frac{1}{t}$ as we know the regret of bandit problems is $\log T$.  %[For RL, it would seem that the similar probabilities should converge at $\frac{1}{t}$ if actions belong to the boundary and $\frac{1}{\sqrt{tt}}$ if they belong to an interior stationary point. Though I don't have a proof for this.] 
Notice the algorithm above optimizes the following objective when we parameterize $p_a$ with a soft-max objective.
\begin{equation*}
	\text{minimize}\quad 
	\sum_{a\in \mathcal A} p_a (r_{a^\star} - r_{a}) \quad \text{subject to} \quad \sum_{a\in \mathcal A} p_a = 1 \quad \text{over} \quad p_a \geq 0,\quad a\in \mathcal A.
\end{equation*}
We can just not reparametrize and apply a gradient descent, taking some care in the step size.  

\smallskip
\noindent \textbf{SAMBA.} Analogous to the soft-max discussion above. This is how to derive a stochastic policy gradient algorithm in this case.  Gradient descent the performs the update 
\[
p_a \leftarrow p_a + \gamma_a (r_a-r_{a^\star})
\]
for $a\neq a^\star$. However, since the mean rewards $r_a, a\in \mathcal A,$ are not known, a stochastic gradient descent must be considered: $p_a \leftarrow p_a + \gamma (R_a-R_{a^\star})$, $a\neq a^\star$. 
Also, the optimal arm is unknown. So instead of $a^\star$, we let $a_\star$ be the arm for which $p_a$ is maximized and, in place, consider the update $p_a \leftarrow p_a + \gamma (R_a-R_{a_\star})$, $a\neq a_\star$.
Since the reward from only one arm can be observed at each step, we apply importance sampling: 
\begin{equation}
\label{p_gamm_update}
p_a \leftarrow p_a + \gamma 
\Big( \frac{R_aI_a}{p_a}-\frac{R_{a_\star}I_{a_\star}}{p_{a_\star}} \Big),\quad 
\text{for }a\neq a_\star\, .
\end{equation}
This gives a simple recursion for a multi-arm bandit problem. A name for this is SAMBA: stochastic approximation multi-arm bandit. Shortly, I'll argue that we need to let $\gamma$ depend on $a$ and we should take $\gamma_a =\alpha p_a^2$ for $\alpha$ suitably small.
Catchy acronyms aside, one can see this is really a stochastic gradient descent algorithm with some correction to make sure we don't get too close to the boundary. A motivation is projected gradient descent or barrier methods in optimization [there is probably a regularization interpretation as well]. %For policy gradients some form of projection/barrier seems to be required to maintain exploration. The general forms that work seems to be an important open problem.\\

 \smallskip 
\noindent \textbf{Learning rate and o.d.e. analysis.} Let's consider the learning rate $\gamma$. Again, consider the gradient descent update $p_a \leftarrow p_a + \gamma (r_a-r_{a^\star})$, for $a\neq a^\star$.
%To show that the regret the grows logarithmically in $t$, it is sufficient to show $p_a(t)$ decreases at a rate $\mathcal O(\frac{1}{t})$, see Lemma \ref{Regret_Lemma}. 
Notice if we let 
$
\gamma = \alpha p_a^2	
$
then the gradient descent algorithm approximately obeys the following ordinary differential equation:
\[
\dot p_a = - \alpha p_a^2 \Delta_a
\]
where, as before, $\Delta_a= r_{a^\star} - r_a$.
We can show the following result.

\begin{theorem} For $p_a(0)=1/N$, $a\in\mathcal A$,
	\[
	 \mathcal{R}\! g(T) = O \left( \sum_{a\neq a^\star} \frac{1}{\alpha \Delta_a} \log T \right)
	\]
\end{theorem}
\begin{proof}
The above o.d.e. has a solution
\[
p_a(t)= \frac{p(0)}{1+\alpha \Delta_ap(0) t} = \frac{1/N}{1+\alpha  \Delta_a t/N}
\]
This implies 
\begin{align*}
\mathcal R\! g(T) = \int_0^T \sum_a (r_{a^\star} -r_a) p_a(t) dt 
&
\leq  
\int_0^T \sum_{a\neq a^\star} p_a(t) dt 
\\
&
\leq 
\sum_{a\neq a^\star} \frac{1}{N} 
\int_0^T \frac{1}{1+\alpha  \Delta_a t/N} dt
\\
&
= 
\sum_{a \neq a^\star}
\frac{1}{\alpha\Delta_a}
 \log ( 1+ \alpha \Delta_a T /N )
\sim 
\sum_{a\neq a^\star} \frac{1}{\alpha \Delta_a} \log T
 .
\end{align*}

%
%\[
%p_a(t) = \frac{p_a(0)}{1+ p_a(0) (r_{a^\star} - r_a) t}
%< \frac{1}{\alpha (r_{a^\star} - r_a) t}\, .
%\] 
This suggest a learning rate of $\gamma = \alpha p_a^2$, applied to each $a$, gives a logarithmic regret. 
\end{proof}
\noindent Notice the above upper-bound is similar to the lower-bound from Lai and Robbins. However, one should be careful to read too much into this as the learning rate $\alpha$ can have a significant impact on the performance of the algorithm when stochastic effects are included.
This is discussed in more detail in the article \cite{DW20}. (Also shorter discussion on discrete time and martingale versions of the above can be found on the weblink below.\footnote{\url{https://appliedprobability.blog/2020/05/21/a-short-discussion-on-policy-gradients-in-bandits/}})

\noindent \hrulefill

\noindent \textbf{Discussion on Convergence Issues.} Both soft-max and SAMBA step rules require some form of best arm identification. For SAMBA this is explicit in that we need $a_{\star}$ the highest probability arm to equal the optimal arm $a^\star$. [And a lot of the technical leg work in the paper involves proving this happens]. For soft-max it is clear that if $p_{a^\star}$ gets small for a sustained period of time then this slows convergence. So in both cases we need $p_{a^\star}$ to get big in a reasonable length of time. This argument is more straight-forward in the o.d.e case where we can bound $p_{a^\star}$ away from zero by a constant. In the stochastic case we need sub-martingale arguments to do this for us. An this can be fiddly as we are essentially dealing with a random walk that is close to threshold between recurrence and transience. 

One thing that seems to come out of the analysis for both soft-max [when you include 2nd order terms] and SAMBA is that if the learning rate is too big then then this random walk switches from being transient [and thus converging on the correct arm] to being recurrent [and thus walking around the interior of the probability simplex within some region of the optimal arm]. One way to deal with this is to slow decrease the learning rate either as a function of either time $\alpha \sim 1/\log t$ or as a function of the state $\alpha \sim 1/(1-\log p_a)$. This appears to multiply on an extra $\log t$ term on the regret bound in both cases while guaranteeing global convergence in the bandit setting. We can consider more slowly decreasing functions which impact regret to an arbitrarily small amount. So it seems like there is a regret of $(\log T)^{1+\epsilon}$ for $\epsilon$ arbitrarily small.

A final point is that in all the analysis so far [both softmax and SAMBA], we have used an o.d.e. of the form
\[
\frac{d x}{dt} = -\alpha(t) x(t)^{2}
\]
which suggests that we apply a Lyapunov function of the form:
\[
\frac{1}{x(t)^\lambda} - \sum_{s=0}^t \alpha(s) \, . 
\]
Notice, in the above expression, we can trade-off between the power applied to the learning rate $\alpha(s) = t^{-\gamma}$ and the power applied to the state $x(s)^{-\lambda}$.
Proving martingales properties for these Lyapunov functions seems to be a key ingredient for getting proofs to work. \\

\noindent \textbf{Acknowledgement.} The author is grateful to Csaba Szepesvari for suggesting to make this note available and to Tor Lattimore for first suggesting Soft-Max as an alternative to the SAMBA o.d.e..

\subsubsection*{Appendix.}

%\bibliographystyle{informs2014} % outcomment this and next line in Case 1
%\bibliography{SAMBA} % if more than one, comma separated

\bibliography{SAMBA}

\end{document}